\documentclass[11pt, letterpaper]{article}
\usepackage{graphicx} 
\usepackage{mathpazo}
\usepackage{amssymb, amsthm, amsmath}
\usepackage{algorithm, algorithmicx, algpseudocode}
\usepackage{truncated}
\usepackage[
letterpaper,
top=1in,
bottom=1in,
left=1in,
right=1in]{geometry}
\usepackage[page,header]{appendix}
\usepackage{titletoc}
\usepackage{natbib}
\title{Efficient Learning of
Truncated Boolean Product Distributions: Influence to the Rescue}	
\author{\textbf{Rohan Chauhan}\footnote{University of California, Irvine \url{rmchauha@uci.edu}}\\ \small University of California, Irvine, \and \textbf{Ioannis Panageas}\footnote{University of California, Irvine, \url{ipanagea@uci.edu}, Supported by NSF grant CCF-2454115}\\ \small University of California, Irvine}
\newtheorem{theorem}{Theorem}[section]
\newtheorem*{theoremfatrel}{Theorem \ref{thm:fat_rel}}

\newtheorem*{theoreminf}{Theorem \ref{thm:influence}}
\newtheorem{lemma}{Lemma}[section]
\newtheorem{assumption}{Assumption}[section]
\newtheorem{proposition}{Proposition}[section]
\newtheorem{corollary}{Corollary}[section]

\newtheorem{definition}{Definition}[section]
\newcommand{\poly}{\mathsf{poly}}
\newtheorem{inftheorem}{Informal Theorem}

\newenvironment{proofthmfatrel}{%
  \proof}{\endproof}
\newenvironment{prooflemfatrel}{%
  \proof}{\endproof}
\newenvironment{proofthmfatabs}{%
  \proof}{\endproof}
\newenvironment{proofthmlower}{%
  \proof}{\endproof}
\newenvironment{proofthminf}{%
  \proof}{\endproof}

\date{}

\begin{document}




\maketitle

\begin{abstract}
Learning the natural parameters $z \in \mathbb{R}^n$ of discrete distributions $\mu_z$ from independent samples constrained to a subset $S \subseteq \{0,1\}^n$ is a foundational challenge in high-dimensional statistics. Existing methods for efficiently estimating truncated Boolean product distributions, notably the work of [Fotakis et al' COLT'20, Algorithmica '22], require either strong local connectivity assumptions on $S$ -- a property denoted \emph{fatness} -- or stringent anti-concentration assumptions and necessitate the total mass of the truncation set to be a constant with respect to $n$. Moreover, the results in [Fotakis et al' COLT'20, Algorithmica '22] suffer from sample complexities that scale as $\Omega(2^n)$ if the mass of $S$ is exponentially small in $n$. 

In this work, we circumvent these limitations by analyzing the geometry of $S$ under the measure $\mu_z$. We refine the existing parameter estimation guarantees under the fatness assumption, improving the prior sample complexity to $\mathcal{O}( \log n / \epsilon^2)$ for $\ell_\infty$-recovery, matching the untruncated minimax rate. We further generalize fatness using the notion of influence utilized in the analysis of Boolean functions and provide sufficient conditions for efficient inference. Notably, unlike previous work, our method does not require sampling at arbitrary parameterizations of the model. Lastly, we establish a theoretical lower bound demonstrating the sample complexity exhibits an intrinsic exponential dependence on the width of the model and the minimum distance between elements in the set. 
\end{abstract}
\newpage 
\newpage
\section{Introduction}

Learning from \emph{truncated samples} is a storied and challenging problem within the field of statistics, wherein the objective is to estimate the model parameters of the underlying true distribution given samples which lie on a subset $S$ of the support of the measure. Truncated samples appear in many fields, such as economics, engineering, biological sciences, and networks, among others, with classical examples including sampling bias when gathering subjects for medical experiments and survivor bias in actuarial analysis. Rigorous statistical estimation from truncated samples dates back at least to 1760, with Daniel Bernoulli's analysis of the efficacy of smallpox treatments \cite{bernoulli1760essai}. This analysis initiated a line of work investigated by Pearson, Galton, and Fisher \citep{galton1898examination, pearson1902systematic, pearson1908generalised, fisher1931properties}, which aimed to develop techniques to robustly estimate and test in truncated environments.

In recent years, there has been a flurry of activity in developing computationally and statistically efficient algorithms for learning both continuous and discrete distributions under truncation. Beginning with the work of \cite{daskalakis2018efficient} which analyzed the proper learning of multivariate Gaussians under a \emph{known} truncation set, provably efficient learning guarantees have been extended to multivariate Gaussians under unknown truncation \cite{kontonis2019efficient}, linear regression with truncated data \cite{daskalakis2019computationally}, and exponential families with known or unknown truncation \cite{lee2023learning, lee2024efficient} as well as other continuous settings. 

Motivated by the numerous examples of complex discrete truncated distributions in genomics \cite{eng2019transcriptome, ghosh2001lateral} and networks \cite{durvy2006packing, zafer2006blocking}, among others, \cite{fotakis2022efficientparameterestimationtruncated} initiated the study of estimating parameters of \textit{discrete} models from truncated samples, developing efficient algorithms for inference in truncated Boolean product distributions. Similar to the techniques used in the aforementioned works concerning truncated estimation in continuous settings, the authors demonstrated that a nontrivial survival mass, \(\mu(S)\ge \alpha\), and a global anti-concentration condition, such as \(\operatorname{Cov}_{\bx\sim\mu}(\bx)\succeq \lambda I\), with \(\alpha,\lambda\) being \(\Theta(1)\), are sufficient for inference. Yet, under these assumptions, the resulting sample complexity can scale as \(\poly(1/\alpha)^{\poly(1/\lambda)}\)\footnote{See Theorem 4 (p. 20) in the arXiv version (arXiv:2007.02392v3) of \cite{fotakis2022efficientparameterestimationtruncated}.}, which becomes prohibitive when \(\lambda=o(1)\) scales with the problem parameters. This stands in stark contrast to the associated lower bound found in \cite{fotakis2022efficientparameterestimationtruncated}, which has no dependence on $\alpha$ and scales on the order of $\poly(1/\lambda)$. Finding efficient algorithms which operate in this regime is a tantalizing open question. Moreover, the associated stochastic gradient methods require estimating gradients not only at the true parameter but throughout a neighborhood of candidate models, which entails rejection sampling from a family of truncated distributions \(\mu_{S}\) for many nearby parameters. Such a requirement is already delicate for product distributions and becomes substantially more untenable for general Markov random fields, where sampling is provably hard.

The algorithmic guarantees of \cite{fotakis2022efficientparameterestimationtruncated} which do not depend on the survival mass rely on local information; given samples $\bx \sim \mu$ that lie in $S$, the authors aim to estimate the underlying parameters via the conditional density $\pr_{\bx \sim \mu}(x_i\mid \bx_{-i})$. The idea of good ``average'' local connectivity for any given sample was formalized by \citep{fotakis2022efficientparameterestimationtruncated} into the notion of \emph{fatness}, which deems a distribution $\gamma$-fat if $\pr_{\bx \sim \mu}((1-x_i, \bx_{-i}) \in S) \ge 1/\poly(n) = \gamma$ for all $i \in [n]$. The authors then demonstrate efficient parameter estimation and sampling under this condition. However, in many sets of interest, such as the set of all Boolean vectors with an even number of ones (the parity set), no element has a single-flip neighbor, which violates the assumption and causes the conditional distributions to collapse.

Remedying this difference, in this work, we examine the task of efficient recovery of the parameters of \emph{truncated} product distributions over the hypercube $\{0,1\}^n$ -- the simplest class of discrete distributions -- given independent and identically distributed (i.i.d.) samples from the underlying model. We wish to develop a framework to analyze these distributions beyond restrictions on the mass of $S$, strong anti-concentration assumptions and single-flip local connectivity, leading to the following question: 
\begin{center}
   \textit{ Are there computationally efficient algorithms that can learn truncated Boolean product distributions without dependence on the survival mass and stringent local connectivity conditions?} 
\end{center}
Our primary contribution in this work is a positive answer to the above question by generalizing the notion of \emph{fatness} to capture a much wider range of truncation sets.

\subsection{Our Results}
Departing from approaches that require sampling from the untruncated distribution or estimating global gradients of the truncated likelihood, we instead exploit the local structure of the truncation set under the product measure. Let \(\mu_{\bp}\) be a Boolean product measure on \(\{0,1\}^n\), that is, a distribution whose coordinates are independent and \(\pr_{\bx \sim \mu_{\bp}}(x_i=1)=p_i\in(0,1)\). We assume oracle access to a truncation set \(S\subseteq\{0,1\}^n\) and observe samples from \(\mu_{\bp\mid S}\), namely, $\mu_{\bp}$ conditioned on $S$. Our primary goal is to compute an estimator \(\widehat{\bp}\) of the true parameter vector \(\bp\) such that, with probability at least \(1-\delta\), \(\|\widehat{\bp}-\bp\|_\infty\le \epsilon\), in time polynomial in \(n\), \(1/\epsilon\), and \(\log(1/\delta)\). We also consider estimation of the corresponding natural parameters \(\bz\), defined coordinatewise by \(z_i=\log(p_i/(1-p_i))\) for \(i\in[n]\). Learning in the \(\bp\)-parameterization corresponds to additive error in the Bernoulli probabilities, whereas learning in the \(\bz\)-parameterization corresponds to controlling errors in the log-odds, and hence gives a relative-error notion for the underlying probabilities.

Our first contribution is improving the sample complexity of estimation under the assumption that the truncated measure $\mu_{\bp\mid S}$ is $\gamma$-fat. Roughly speaking, \(\mu_{\bp\mid S}\) is \(\gamma\)-fat in coordinate \(i\) if, with probability at least \(\gamma\) over \(\bx\sim\mu_{\bp\mid S}\), the neighboring point \((1 - x_i, \bx_{-i})\) also lies in \(S\). This implies that on a \(\gamma\)-fraction of truncated samples, both values of the \(i\)-th coordinate are feasible while the remaining coordinates are fixed, so the conditional distribution of \(x_i\) reveals the corresponding one-dimensional marginal information of the original product distribution. In this sense, fatness is a conditional internal-connectivity condition on \(S\).

\begin{inftheorem}[Learning under $\gamma$-fatness]
    Let \(\bx^{(1)},\dots,\bx^{(N)}\) be i.i.d. samples from a $\gamma$-fat truncated Boolean product distribution $\mu$, and suppose $\|\bz\|_\infty\le R$. There are polynomial-time algorithms which estimate the probability vector to $\ell_\infty$-error $\epsilon$ from
    \[
    N=\mathcal O\left(\frac{\log(n/\delta)}{\gamma\epsilon^2}\right)
    \]
    samples and estimate the natural parameter vector to $\ell_\infty$-error $\epsilon$ from
    \[
    N=\mathcal O\left(\frac{e^{2R}\log(n/\delta)}{\gamma\epsilon^2}\right)
    \]
    samples, each with probability at least $1-\delta$.
\end{inftheorem} 

This sample complexity improves over previous work by a factor of $\log(n)$ and matches the minimax rate for estimating truncated Boolean product distributions when $\gamma \in \Theta(1)$, since estimating each coordinate to accuracy \(\epsilon\) requires order \(1/\epsilon^2\) samples and a \(\log(n/\delta)\) factor for uniform \(\ell_\infty\) control. Moreover, this bound depends only on the local fatness parameter \(\gamma\), and has no dependence on the global survival mass \(\mu_{\bp}(S)\). Assuming fatness, these algorithms give polynomial sample complexity even when the aforementioned anti-concentration parameter $\lambda \in o(1)$\footnote{The relationship between the width of the model $R$ and the anti-concentration parameter of \cite{fotakis2022efficientparameterestimationtruncated} is expounded on in Appendix \ref{app:fotakis}.}. 

However, the notion of fatness is too restrictive for many natural truncation sets. For example, the parity set \(S=\{x\in\{0,1\}^n:\sum_i x_i\equiv 0 \pmod 2\}\) has no single-coordinate connectivity. For every \(i \in [n], \bx \in S\), \(\pr_{\bx\sim\mu_{\bz\mid S}}((1 - x_i,\bx_{-i})\in S)=0\) implying any analysis based solely on single-bit flips is vacuous. Nevertheless, parity is preserved by two-bit flips, and conditioning on all coordinates outside \(\{i,j\}\) yields a nontrivial one-dimensional Bernoulli problem whose natural parameter is one of \(z_i+z_j\) or \(z_i-z_j\). This suggests replacing direct coordinate-wise estimation of \(z_i\) with the estimation of sparse linear forms \(\bw_I^\top \bz\), where \(\bw_I\in\{-1,0,1\}^n\) is supported on a small coordinate set \(I\). Formalizing the existence of such higher-order moves, we use the following notion of \emph{conditional influence}.
\begin{definition}[Conditional Influence]
\label{def:inf}
Given a truncated Boolean product distribution \(\mu_{\bz\mid S}\) and a coordinate set \(I\subseteq[n]\), define the conditional influence of \(I\) by
\[
\Inf_I^{\,\mu_{\bz\mid S}}
:=
\Pr_{\bx\sim\mu_{\bz\mid S}}\!\left[\bx^{\oplus I}\notin S\right],
\]
where \(\bx^{\oplus I}=(1 - \bx_I,\bx_{-I})\) denotes the vector obtained from \(x\) by flipping all coordinates in \(I\).
\end{definition}

Thus, \(\Inf_I^{\,\mu_{\bz\mid S}}\) is the probability that a sample from the truncated distribution becomes infeasible after the simultaneous flip of the coordinates in \(I\); equivalently, the feasibility probability of this flip is \(1-\Inf_I^{\,\mu_{\bz\mid S}}\). We recall that the conventional notion of Boolean influence records the probability that the value of a Boolean function changes under a specified flip. In the present setting, Definition~\ref{def:inf} is the one-sided version of this notion for the indicator of $S$, conditional on the original point lying in $S$.

Moreover, $\gamma$-fatness is equivalent to
\(\Inf_{\{i\}}^{\,\mu_{\bz\mid S}}\le 1-\gamma\) for every coordinate \(i\). Our framework replaces this single-coordinate requirement by a higher-order condition: even if every single-coordinate flip is infeasible, learning may remain possible provided that sufficiently many small sets \(I\) have feasibility probability \(1-\Inf_I^{\,\mu_{\bz\mid S}}\) bounded away from zero and the associated vectors \(\bw_I\) span \(\mathbb R^n\) in a quantitatively stable way.

For the model to be identifiable (see Assumption~1 of \cite{fotakis2022efficientparameterestimationtruncated}), we require that the affine span of $S$ is $\mathbb R^n$. This implies that there are feasible flip directions spanning $\mathbb R^n$ (see Lemma~\ref{lem:affine-span-nonzero-influence-pd}). Without any quantitative condition beyond identifiability, however, the corresponding feasibility probabilities can be exponentially small and the signed design can be arbitrarily ill-conditioned. We therefore impose two structural assumptions. The first ensures that a family of coordinate sets $\calF$ gives minimal feasible flips with probability at least $\gamma$. The second, signed anti-concentration, ensures that the induced quantities \(\bw_I^\top\bz\) collectively identify every direction of the parameter vector in a quantitatively stable manner.

Under these assumptions, the inference problem decomposes into three primary steps. First, using samples, we determine which sets of indices are feasible. Then, we select the sign of these feasible indices to provide estimates which are assembled into the linear inverse problem \(\widehat \by = W\bz\). If \(W\) is well conditioned, solving this system yields a stable estimate \(\widehat \bz\). In this sense, higher-order influence generalizes fatness by replacing direct access to individual coordinates with stable access to sufficiently many local signed linear measurements of the natural parameter. 

\begin{inftheorem}
For fixed \(k\), let \(\mathcal F\subseteq\{I\subseteq[n]:1\le |I|\le k\}\) satisfy the minimal-flip and signed-design assumptions with parameters $\gamma$ and $\lambda_s$, and suppose $\|\bz\|_\infty\le R$. There is a polynomial-time algorithm which, from
\[
N
\ge
\mathcal O\left(
\frac{2^k e^{2kR}}{\gamma\lambda_s\epsilon^2}
\left(k\log(2n)+\log\frac{4k}{\delta}\right)
\right)
\]
i.i.d. samples from $\mu_{\bz\mid S}$, returns $\hat{\bz}$ satisfying
\(\|\bz-\hat{\bz}\|_2\le\epsilon\) with probability at least \(1-\delta\).
\end{inftheorem}

We again note that our result does not depend on the survival mass of the set $S$, and instead leverages information found in the neighborhood of Hamming distance\footnote{The Hamming distance metric counts the number of differing indices between two vectors. Thus, a neighbor at Hamming distance $k$ is obtained by flipping exactly $k$ indices of $\bx$.} $k$. The polynomial sample requirement derives from the regression problem of using the sparse linear forms to estimate the underlying parameters, and depends on the condition number of the matrix $W = \begin{bmatrix}
    w_1 & w_2 & \cdots & w_m
\end{bmatrix}$; this is discussed in greater detail in Section~\ref{sec:influence}. Moreover, our sample complexity remains polynomial as long as $R \in \Theta(\log(n))$, making our methods effective for truncated Boolean product distributions beyond constant width.

Finally, we establish an information-theoretic lower bound showing that when \(\operatorname{Inf}_{I}^{\,\mu}(S)=1\) uniformly over all subsets \(I\subseteq[n]\) with \(|I|=k\), the sample complexity of learning necessarily scales exponentially in \(k\).

\begin{inftheorem}
    For every separation scale \(k\) with \(2k\mid n\), there exists a truncation set \(S\subseteq\{0,1\}^n\) whose distinct elements have Hamming distance at least \(k\), and hence no nonempty flip of fewer than \(k\) coordinates is feasible. For this set and \(0<\epsilon\le R\sqrt n/16\), there is a family of parameter vectors \(\bz^*\) with \(\|\bz^*\|_\infty\le R\) such that any estimator achieving \(\|\hat\bz-\bz^*\|_2\le\epsilon\) with constant probability requires
    \[
    N\ge \Omega\left(\frac{n\exp(kR/2)}{k^2\epsilon^2}\right)
    \]
    samples.
\end{inftheorem} 

The lower-bound construction consists of independent blocks of coordinates, each containing only a few admissible configurations, with any two admissible configurations separated in Hamming distance by \(\Omega(k)\). When the true parameter \(\bz^*\) is aligned with one admissible configuration, the conditional distribution places nearly all its mass on that configuration, while each informative alternative has probability at most \(\exp(-\Theta(kR))\). Samples therefore rarely visit the configurations required to distinguish nearby parameters. The obstruction is the local geometry: if transitions between feasible configurations require changing \(k\) coordinates and each coordinate contributes energy up to \(R\), then the effective barrier is \(kR\), yielding an \(\exp(\Theta(kR))\) sample-complexity penalty.
\subsection{Related Work}

Inference and estimation under truncation have long been subjects of interest in statistics \citep{galton1898examination, pearson1902systematic, pearson1908generalised, fisher1931properties}. More recently, numerous efficient algorithms have been developed for statistical learning in this challenging setting, from learning Gaussians \citep{daskalakis2018efficient, kontonis2019efficient, diakonikolas2024statistical}, to exponential families \cite{hannon1999estimation, lee2023learning, lee2024efficient, lee2025learning, karatapanis2025oracle}, linear regression under truncation \cite{daskalakis2019computationally, daskalakis2021efficient, kouridakis2026linear}, non-parametric density estimation \cite{daskalakis2021statistical}, mixtures of Gaussians \cite{nagarajan2020analysis, nagarajan2023mean, tai2023learning}, the Ising model \cite{chauhan2026learning}, and, of particular interest to this work, Boolean product distributions \cite{fotakis2022efficientparameterestimationtruncated, galanis2024learning, galanis2025oneshotlearningksat}, which have formalized the single-flip local connectivity of samples as \emph{fatness} and \emph{flippability}, respectively. Many of these works leverage projected stochastic gradient descent to learn the parameters of the model and utilize rejection sampling to compute an unbiased estimate of the gradient.

In the realm of discrete distributions and graphical models, there has been a parallel line of work leveraging logistic regression for parameter estimation in the Ising model \cite{wu2019sparselogisticregressionlearns, gaitonde2023unifiedapproachlearningising, klivans2017learninggraphicalmodelsusing, chandrasekaran2025learning} and higher-order Markov random fields \cite{zhang2020privately, gaitonde2024bypassingnoisyparitybarrier}. More generally, there have been works on estimating product distributions truncated to the hard-core gas model and graph colorings \cite{bhattacharya2021parameterestimationundirectedgraphical, galanis2024learning}, as well as learning constraint-satisfaction problems from positive examples \cite{feng2025learning}. Likewise, the literature on the analysis of Boolean functions over the hypercube is fundamental to theoretical computer science \cite{o2014analysis}. The notion of influence was introduced by \cite{penrose1946elementary} in the context of genetics and brought to the computer science literature by \cite{ben1987collective}. Building on this, \cite{kahn1989influence} connected the notion to harmonic analysis on the hypercube, thereby demonstrating that ``Boolean functions always have small sets of variables.'' More recently, \cite{diakonikolas2010bounding} investigated the influence of bounded-degree polynomial truncation functions over the uniform hypercube, \cite{KELLER_2011} studied the influence of biased distributions over the hypercube, and \cite{biswas2023influence} developed a novel formulation of influence to better represent multi-bit flips in applications to cryptography.

\section{Preliminaries}
\subsection{Notation}
We denote the set $\{1,2,\ldots,n\}$ by $[n]$. Vectors $\bx \in \R^n$ are denoted with boldface, and matrices $M \in \R^{m \times n}$ with capital letters. Given a vector $\mathbf{a} = (a_1, a_2, \dots , a_n)$ and a subset $I \subseteq [n]$, let $\mathbf{a}_I$ denote the length-$|I|$ coordinate vector $\{a_i : i \in I\}$, and $\mathbf{a}_{-i}$ denote the vector $\mathbf{a}$ with the $i$-th element removed. For $\bx\in\{0,1\}^n$, define its signed encoding by
\[
\bs(\bx):=2\bx-\mathbf 1\in\{\pm1\}^n.
\]
Thus, $s_i(\bx)=2x_i-1$ throughout the paper. We write $\bx^{\oplus I}$ for the vector obtained from \(\bx\) by flipping the coordinates in \(I\), i.e., \((x^{\oplus I})_j = 1-x_j\) if \(j \in I\), and \((x^{\oplus I})_j = x_j\) otherwise. We denote the family of nonempty coordinate subsets of size at most $k$ by $\calI_{\le k} := \{ I \subseteq [n]: 1 \le |I| \le k\}$.

\subsection{Boolean Product Distributions}
In this work, we study \emph{Boolean product distributions} on the binary hypercube $\Pi_n=\{0,1\}^n$, parameterized by $\bp=(p_1,\ldots,p_n)$. This model is the product of $n$ independent Bernoulli distributions\footnote{A Bernoulli distribution with parameter $p\in[0,1]$ assigns probability $p^x(1-p)^{1-x}$ to $x\in\{0,1\}$.}, $\mathcal{B}(p_1)\otimes\cdots\otimes\mathcal{B}(p_n)$, and can be written as the following exponential family:
\[
\pr_{\mu_{\bz}}(\bx) = \frac{\exp(\bx^\top \bz)}{\prod_{i \in [n]}(1 + \exp(z_i))} \tag{Boolean Product Distribution},
\] where $\bz=(z_1,\ldots,z_n)$ is the natural parameter vector of the model, with $z_i=\log(p_i/(1-p_i))$. Moreover, $\sigma(\bz)=\bp$, where $\sigma(x)=1/(1+\exp(-x))$ is the sigmoid function. Given a
subset $S\subseteq\{0,1\}^n$ of the Boolean hypercube, the mass that $\mu_{\bz}$ assigns to $S$ is $\mu_{\bz}(S)=
\sum_{\bx \in S} \mu_{\bz}(\bx)$. Throughout this work, we refer to the Boolean product distribution interchangeably as
$\mu_{\bp}$, $\mu_{\bz}$ and $\mu$, depending on whether the probability parameter $\bp$ or the natural parameter $\bz$ is the focus or understood by context.

Given a Boolean product distribution $\mu_{\bz}$ and a set $S\subseteq\{0,1\}^n$ with $\mu_{\bz}(S)>0$, we define the \emph{truncated} Boolean product distribution by
\[
\mu_{\bz\mid S}(\bx)
:=
\frac{\mu_{\bz}(\bx)\mathbf{1}\{\bx\in S\}}{\mu_{\bz}(S)}.
\]
We refer to $S$ as the \emph{truncation set}. We say that $\mu_{\bz}$ is \emph{identifiable} from $\mu_{\bz\mid S}$ if the map $\bz\mapsto\mu_{\bz\mid S}$ is injective. Likewise, $\mu_{\bz}$ is \emph{efficiently learnable} from $\mu_{\bz\mid S}$ if there is an algorithm which, using samples only from $\mu_{\bz\mid S}$, returns an estimate $\hat{\bz}$ satisfying $\|\bz-\hat{\bz}\|_2\le\epsilon$ with probability at least $1-\delta$, in time and sample complexity polynomial in $n$, $1/\epsilon$, and $\log(1/\delta)$.

\subsection{Distances between Distributions}
Let $\calP, \calQ$ be probability measures over a discrete probability space $(\Omega, \calF)$. Two quantities of interest are the \emph{total variation distance},
\[
D_{\mathrm{TV}}(\calP,\calQ)
:=
\frac12\sum_{\bx\in\Omega}|\calP(\bx)-\calQ(\bx)|,
\]
and the \emph{Kullback--Leibler divergence},
\[
D_{\mathrm{KL}}(\calP\|\calQ)
:=
\ex_{\bx\sim\calP}\left[\log\frac{\calP(\bx)}{\calQ(\bx)}\right].
\]
For Boolean product distributions, both quantities can be controlled directly through the natural parameters.
\begin{proposition}
    Let $\mu_{\bz}$ and $\mu_{\bz'}$ be two Boolean product distributions with natural parameters $\bz$ and $\bz'$, respectively. Then
\begin{itemize}
    \item $D_{\mathrm{KL}}(\mu_{\bz}\|\mu_{\bz'}) \le \frac18\|\bz-\bz'\|_2^2$.
    \item $D_{\mathrm{TV}}(\mu_{\bz},\mu_{\bz'}) \le \frac14\|\bz-\bz'\|_2$.
\end{itemize}
\end{proposition}
\begin{proof}
The log-partition function of the product family is
$A(\bz)=\sum_{i=1}^n\log(1+e^{z_i})$, and
\[
\nabla^2 A(\bz)
=
\operatorname{diag}\bigl(\sigma(z_i)(1-\sigma(z_i)):i\in[n]\bigr)
\preceq \frac14 I_n.
\]
The KL divergence is the Bregman divergence
$A(\bz')-A(\bz)-\langle\nabla A(\bz),\bz'-\bz\rangle$.
Taylor's theorem therefore gives the first claim. The second follows from
Pinsker's inequality,
$D_{\mathrm{TV}}(\calP,\calQ)\le
\sqrt{D_{\mathrm{KL}}(\calP\|\calQ)/2}$.
\end{proof}

\section{Boolean Product Distributions Truncated by $\gamma$-Fat Sets \label{sec:fat}}
As a warm-up to the techniques we introduce for learning truncated Boolean product distributions, we study how the notion of \emph{fatness} introduced by \citep{fotakis2021efficient} can guarantee efficient parameter estimation. We give the formal definition as follows.
\begin{assumption}[Fatness]
    A truncated distribution $\mu_{\bz\mid S}$ is $\gamma$-fat in coordinate $i$ if \[
    \pr_{\bx \sim \mu_{\bz\mid S}}(\bx^{\oplus i} \in S )  \ge \gamma.
    \] We call a distribution $\mu_{\bz\mid S}$ $\gamma$-fat if it is $\gamma$-fat in every coordinate.
\end{assumption}
\emph{Remark:} If $\mu_{\bz\mid S}$ is fat in coordinate $i$, then both $(\bx_{-i},0)$ and $(\bx_{-i},1)$ lie in the truncation set $S$ with nontrivial probability. This notion is already quite general and includes halfspaces $S_{\le k}=\{\bx\in\{0,1\}^n:\sum_{i=1}^n x_i\le k\}$ and \emph{downward-closed} sets, among others.

We give efficient parameter estimation algorithms for $\bz$ and $\bp$ under this assumption. 
\begin{theorem}[Learning $\mu$ in Relative Error]
\label{thm:fat_rel}
    Let $\bx^{(1)},\ldots,\bx^{(N)}$ be independent random variables drawn from a $\gamma$-fat distribution $\mu_{\bz\mid S}$, where $\|\bz\|_\infty\le R$. For every $\epsilon,\delta\in(0,1)$, if
    \[
    N\ge
    C\frac{n e^{2R}}{\gamma\epsilon^2}
    \log\left(\frac{2n}{\delta}\right)
    \]
    for a universal constant $C>0$, then there is an algorithm that returns an estimate $\hat{\bz}$ satisfying
    $\|\bz-\hat{\bz}\|_2\le\epsilon$ with probability at least $1-\delta$.
\end{theorem}
The relationship between natural-parameter error and distributional distance yields the following corollary for learning in total variation distance.
\begin{corollary}[Learning $\mu$ in Total Variation Distance]
    Under the assumptions of Theorem~\ref{thm:fat_rel}, there is a universal constant $C>0$ such that
    \[
    N\ge
    C\frac{n e^{2R}}{\gamma\epsilon^2}
    \log\left(\frac{2n}{\delta}\right)
    \]
    samples suffice to return a Boolean product distribution $\hat\mu$ satisfying
    $D_{\mathrm{TV}}(\mu_{\bz},\hat\mu)\le\epsilon$ with probability at least $1-\delta$.
\end{corollary}
\begin{proof}
For $\epsilon\le1/4$, apply Theorem~\ref{thm:fat_rel} with target natural-parameter error $4\epsilon$ and set $\hat\mu:=\mu_{\hat{\bz}}$. The proposition in Section~2 gives
\[
D_{\mathrm{TV}}(\mu_{\bz},\hat\mu)
\le
\frac14\|\bz-\hat{\bz}\|_2
\le
\epsilon.
\]
Absorbing the numerical factor into $C$ gives the stated sample bound. For $\epsilon>1/4$, run Theorem~\ref{thm:fat_rel} with natural-parameter target error $1/2$; this yields TV error at most $1/8<\epsilon$, and the stated sample bound follows after increasing $C$.
\end{proof}
Lastly, we give estimation results for learning the underlying probabilities $\bp$; note that the sample complexity depends only on $\epsilon$ and the fatness of the distribution.
\begin{theorem}[Learning $\mu$ in Absolute Error]
\label{thm:fat_abs}
    Let $\bx^{(1)},\ldots,\bx^{(N)}$ be independent random variables drawn from a $\gamma$-fat distribution $\mu_{\bp\mid S}$. For every $\epsilon,\delta\in(0,1)$, there is an algorithm returning $\hat\bp\in[0,1]^n$ such that $\|\bp-\hat\bp\|_2\le\epsilon$ with probability at least $1-\delta$, provided
    \[
    N\ge
    C\frac{n}{\gamma\epsilon^2}
    \log\left(\frac{2n}{\delta}\right).
    \]
    Moreover, the same algorithm satisfies $\|\bp-\hat\bp\|_\infty\le\epsilon$ with probability at least $1-\delta$, provided
    \[
    N\ge
    C\frac{1}{\gamma\epsilon^2}
    \log\left(\frac{2n}{\delta}\right).
    \]
\end{theorem}
We next give a brief overview of the algorithms used to find these estimators and the theoretical analysis underpinning their success, deferring additional details to the appendix.
\subsection{Learning $\mu$ in Relative Error}
In order to learn the underlying natural parameter $\bz$, we utilize a one-dimensional version of the logistic regression framework introduced by \cite{wu2019sparselogisticregressionlearns}. For each coordinate $i$, let
\[
\calJ_i:=\{j\in[N]:(\bx^{(j)})^{\oplus i}\in S\},
\qquad
M_i:=|\calJ_i|,
\qquad
y_i^{(j)}:=2x_i^{(j)}-1.
\]
When $M_i>0$, the estimator minimizes the conditional empirical loss
\[
\hat z_i
\in
\operatorname*{arg\,min}_{w\in[-R,R]}
\widehat{\calL}_i(w),
\qquad
\widehat{\calL}_i(w)
:=
\frac1{M_i}\sum_{j\in\calJ_i}
\log\left(1+\exp(-y_i^{(j)}w)\right).
\]
Conditional on a sample being flippable in coordinate $i$, its signed coordinate $y_i$ has logistic natural parameter $z_i$. Hence the corresponding population loss is minimized at $z_i$.

To demonstrate that the minimizer of the empirical loss is close to the population minimizer, we show two key facts for all $i\in[n]$:
\begin{itemize}
    \item Conditional on flippability, $\min_{w\in[-R,R]}\nabla^2\calL_i(w)\ge\Omega(e^{-R})$; the factor $\gamma$ does not enter this conditional curvature.
    \item A Chernoff bound gives $M_i\ge\gamma N/2$, and conditional on this event Hoeffding's inequality ensures
    $\nabla\widehat\calL_i(z_i-\epsilon)<0<
    \nabla\widehat\calL_i(z_i+\epsilon)$ whenever the displayed endpoints lie in $[-R,R]$.
\end{itemize}
Strong convexity provides a gradient signal of order $e^{-R}\epsilon$ at distance $\epsilon$ from the true parameter. Concentration on the flippable subsample then shows that the empirical minimizer lies within $\epsilon$ of $z_i$. Since the flippable subsample has order $\gamma N$ observations, the resulting dependence is $1/\gamma$, rather than $1/\gamma^2$.

To find $\hat \bz$, the algorithm estimates each coordinate \(z_i\) independently by using only samples flippable in coordinate $i$, minimizing the resulting one-dimensional convex empirical logistic loss via bisection on its derivative, and returning the coordinate-wise minimizers.

\subsection{Learning $\mu$ in Absolute Error}

To learn the probabilities $\bp$ of the underlying model, we \emph{simply} take the sample average of each coordinate over all samples that have that index flippable. On such
samples, both values of the \(i\)-th coordinate are feasible while the other
coordinates are fixed, so the conditional distribution of \(x_i\) is still
Bernoulli with parameter \(p_i\), and the empirical average of \(x_i\) is an unbiased estimator of \(p_i\). Repeating this
procedure independently for each coordinate gives an estimate
\(\hat{\bp}=(\hat p_1,\dots,\hat p_n)\). The fatness assumption
ensures that each coordinate has sufficiently many flippable samples.

\section{Estimation Beyond Fatness \label{sec:influence}}
Our primary idea for learning guarantees that do not rely on the survival mass of the underlying domain is to extend fatness from single-coordinate flips to multi-bit flips. In this sense, the generalized notion of fatness measures internal connectivity with respect to \(S\) under the measure $\mu_{\bz\mid S}$. For estimation, we consider feasible moves along subsets of coordinates and learn linear forms of the natural parameter vector. Specifically, for vectors \(\bw\in\{-1,0,1\}^n\) with sparse support \(I=\operatorname{supp}(\bw)\), we aim to estimate quantities of the form \(\bw^\top\bz\). The condition \(\Inf_I^{\,\mu}<1\) asserts that the truncation set has nontrivial internal variation along the coordinates in \(I\), implying that the conditional distribution of $\bw_I^\top\bx_I$ given $\bx_{-I}$ yields information about the corresponding linear form. If we collect such estimates for a family of directions \(\bw_1,\dots,\bw_m\), and denote \(y_\ell=\bw_\ell^\top\bz\) for \(\ell=1,\dots,m\), then the problem reduces to recovering \(\bz\) from an approximately observed linear system \(\by=W\bz\), where \(W\in\mathbb{R}^{m\times n}\) has rows \(\bw_\ell^\top\). If \(W\) is well conditioned on \(\mathbb{R}^n\), then accurate estimates \(\widehat y_\ell\) of the linear forms yield an accurate estimate of \(\bz\) by solving the inverse problem. For example, when \(m=n\) and \(W\) is invertible, one may take \(\widehat{\bz}=W^{-1}\widehat{\by}\), and the resulting parameter error is controlled by the conditioning of \(W\). With this framework in mind, we now demonstrate structural conditions on the influences $\{\Inf_I^{\,\mu_{\bz\mid S}}\}_{I\subseteq[n]}$ and the underlying set $S$ that provide estimation guarantees.

The first step is to guarantee the existence of local two-point conditional distributions which can be reliably estimated by probing minimal feasible flips. For a nonempty $I\subseteq[n]$, define
\[
B_I
:=
\left\{
\bx^{\oplus I}\in S
\text{ and }
\bx^{\oplus J}\notin S
\text{ for every }\varnothing\neq J\subsetneq I
\right\}.
\]
On $B_I$, the only feasible configurations obtained by flipping a subset of $I$ are $\bx$ and $\bx^{\oplus I}$. Thus, after conditioning on the coordinates outside $I$, these two orientations form a Bernoulli experiment. The following condition ensures that this minimal-flip event has probability at least $\gamma$.

\begin{assumption}[Bounded Influences]
\label{assmpt:minimal}
Let $\calF\subseteq \{I\subseteq[n]:1\le |I|\le k\}$. We assume there exists $\gamma>0$ such that, for every $I\in\calF$,
\[
\Inf_I^{\,\mu_{\bz\mid S}}\le 1-2\gamma,
\qquad
\Inf_J^{\,\mu_{\bz\mid S}}
\ge
1-\frac{\gamma}{2^{|I|}-2}
\quad
\forall \emptyset\neq J\subsetneq I,
\]
where the second condition is vacuous when $|I|=1$.
\end{assumption}

For $I\in\calF$ and $t\in\{\pm1\}^I$, define the joint one-sided mass
\[
m_I(t):=\Pr(B_I,\ \bs_I(\bx)=t).
\]
Choose $t_I$ maximizing $m_I(t)$, and let $s_I$ be the canonical representative of the antipodal pair $\{t_I,-t_I\}$, with $(s_I)_{\min I}=+1$. Extend $s_I$ by zero outside $I$ to obtain $\bw_I\in\{0,\pm1\}^n$. The choice between $s_I$ and $-s_I$ is immaterial below because only $\bw_I\bw_I^\top$ appears.

We next require that these signed directions contain enough geometric information to recover the full parameter vector $\bz$. Without such a condition, the accessible linear forms may fail to identify certain coordinates; for example, the available measurements may reveal only $z_i+z_j$, rather than the two coordinates separately. Even when identifiability holds, recovery may be unstable if the signed design matrix is poorly conditioned.
\begin{assumption}[Flip Directions]
\label{asmp:anti}
There is a parameter $\lambda_s>0$ such that, for every $v\in\R^n$,
\[
\frac{1}{|\calF|}
\sum_{I\in\calF}\langle \bw_I,v\rangle^2
\ge
\lambda_s\|v\|_2^2.
\]
\end{assumption}

The quantity $\lambda_s$ measures the least visible direction of $\bz$ under
the available linear forms, a control over the sample complexity required to solve the inverse problem. 

\begin{theorem}
\label{thm:influence}
Fix $k$, and let $\bx^{(1)},\ldots,\bx^{(N)}$ be independent random variables drawn from $\mu_{\bz\mid S}$, where $\mu_{\bz\mid S}$ satisfies Assumptions~\ref{assmpt:minimal} and~\ref{asmp:anti}, and $\|\bz\|_\infty\le R$. There is a universal constant $C>0$ such that, for every $\epsilon,\delta\in(0,1)$, if
\[
N
\ge
C\frac{2^k e^{2kR}}{\gamma\lambda_s\epsilon^2}
\left(
k\log(2n)+\log\frac{4k}{\delta}
\right),
\]
then a polynomial-time algorithm returns an estimate $\widehat{\bz}$ satisfying
$\|\widehat{\bz}-\bz\|_2\le\epsilon$ with probability at least $1-\delta$.
\end{theorem}
\emph{Remark:}
The support-estimation step enumerates all nonempty $I\subseteq[n]$ with $|I|\le k$, and therefore for constant $k$ its runtime is $n^{O(k)}$, up to the cost of membership queries and logarithmic factors. The remaining steps are polynomial provided the relevant conditioning parameters are not too small. In particular, if $\gamma\ge 1/\operatorname{poly}(n)$ and $\lambda_s\ge 1/\operatorname{poly}(n)$, then the sample complexity and runtime are polynomial in $n$, $1/\epsilon$, and $\log(1/\delta)$ for constant $k$ and $R=O(1)$. Since each signed row satisfies $\|\bw_I\|_2^2=|I|\le k$, the population signed Gram matrix has trace at most $k$, and hence average eigenvalue at most $k/n$. Thus the assumption $\lambda_s\ge 1/\operatorname{poly}(n)$ is the standard polynomial-conditioning requirement for the inverse problem. We further note that the exponential dependence on $k$ arises from searching over and estimating signed flips of size at most $k$, and we show in the next section that this dependence is endemic through a corresponding lower bound.
\subsection{Proof Sketch}
Before describing the proof, we make precise which signed directions
are learned.  For every nonempty $I\subseteq[n]$ with $|I|\le k$, we
pair each sign pattern $s\in\{\pm1\}^I$ with its antipode $-s$ and keep
one canonical representative.  We write $\bw_{I,s}\in\{0,\pm1\}^n$ for
the extension of $s$ by zero outside $I$.  A sample contributes to the
pair $(I,s)$ when its signed coordinates on $I$ have orientation $s$ or
$-s$, the full $I$-flip remains in $S$, and every proper nonempty
subflip leaves $S$.  We denote this paired event by $E_{I,s}$ and its
probability by $p_{I,s}$.

The bounded-influence assumption guarantees that the minimal-flip event
for every $I\in\calF$ has probability at least $\gamma$.  Pairing
antipodal sign patterns and applying the pigeonhole principle shows that
the selected pair $(I,s_I)$ has mass at least
$\gamma/2^{|I|}\ge\gamma/2^k$.  The flip-direction assumption then says
that these selected signed vectors collectively see every direction in
parameter space.  Equivalently, there is a well-conditioned
distribution over sufficiently frequent candidate pairs.  This
distribution is used only in the analysis; the learner does not need to
know it.

The proof is organized around two main lemmas.
\begin{itemize}
    \item Lemma~\ref{lem:uniform-data-event} shows that the learner
    finds all informative signed directions and estimates their
    corresponding linear forms uniformly well.
    \item Lemma~\ref{lem:designed-recovery} shows that these estimated
    directions form a sufficiently well-conditioned linear system and
    that weighted least squares stably recovers $\bz$.
\end{itemize}

For the first lemma, the learner enumerates the candidate pairs,
estimates the mass $p_{I,s}$ of each paired event, and retains those
whose empirical mass exceeds the threshold used in Appendix~C.  The
threshold has two roles: it keeps every pair from the hidden
well-conditioned family and discards pairs whose true mass is too small
to estimate reliably.  Conditional on a retained paired event, flipping
all coordinates in $I$ exchanges the two orientations $s$ and $-s$.
Their density ratio depends only on the signed linear form
$y_{I,s}:=\langle\bw_{I,s},\bz\rangle$.  In particular, if $A$ records
the observed orientation, then
\[
\Pr(A=a)
=
\frac{e^{a y_{I,s}/2}}{2\cosh(y_{I,s}/2)},
\qquad
a\in\{\pm1\}.
\]
Thus every retained direction gives a one-dimensional logistic
estimation problem.  A Chernoff bound supplies enough observations from
each retained event, and the curvature of this logistic loss gives a
uniform estimate of all retained linear forms.  A union bound over the
at most $k(2n)^k$ candidate pairs produces the first main lemma.

For the second lemma, observe that once all pairs in the hidden
well-conditioned family have been retained, their population weights
give a feasible design over the observed directions.  The learner finds
such weights by a semidefinite feasibility problem and then performs
weighted least squares on the estimated linear forms.  The Gram lower
bound prevents errors in the individual linear forms from being
amplified in poorly observed directions.  The deterministic stability
argument in Appendix~C yields
\[
\|\widehat{\bz}-\bz\|_2
\le
\eta\sqrt{\frac{2}{\lambda_s}},
\]
where $\eta$ is the uniform error in the estimated linear forms.

Finally, taking $\eta=\epsilon\sqrt{\lambda_s/2}$ and using the lower
bound $\gamma/2^k$ on the mass of each informative pair gives the sample
complexity in Theorem~\ref{thm:influence}.  For constant $k$, enumerating
candidate pairs, solving the one-dimensional logistic problems, finding
the design weights, and performing weighted least squares all take
polynomial time.

\section{Lower Bound on the Sample Complexity of Estimation under Truncation \label{sec:lower}}
In this section, we prove that the sample complexity of estimating truncated Boolean product distributions depends exponentially on the \emph{local connectivity} of the space, namely, the average number of neighbors of a sampled point.

\begin{theorem}[Sample Complexity Lower Bound]
\label{thm:lower_bound}
Fix integers $n,k$ such that $2k\mid n$, and let $R>0$ and
$0<\epsilon\le R\sqrt n/16$. There is a truncation set
$S\subseteq\{0,1\}^n$ whose distinct elements have Hamming distance at
least $k$ such that the following holds. Any estimator which, for every
$\bz$ in the constructed family with $\|\bz\|_\infty\le R$, satisfies
$\|\widehat{\bz}-\bz\|_2\le\epsilon$ with probability at least $0.99$
requires
\[
N
\ge
c\frac{n\exp(kR/2)}{k^2\epsilon^2}
\]
samples, where $c>0$ is a universal constant.
\end{theorem} 
\emph{Remark:} The lower bound demonstrates the intrinsic hardness of this construction is
governed by the product \(kR\): the block separation parameter \(k\) controls
how many coordinates must change to move between admissible states, while the
parameter bound \(R\) controls how strongly the distribution can concentrate on
any given state. Together, they produce an exponentially small
probability of observing informative alternative block states, of order
\(\exp(-\Theta(kR))\), and hence an \(\exp(\Theta(kR))\) contribution to the
sample complexity. In particular, if \(kR\in \Theta(\log n)\), then this
exponential factor is only polynomial in \(n\), so the lower bound remains
polynomial.

\subsection{Construction of the Adversarial $S$}
We construct the truncation set \(S\) by first defining a small, highly separated
set of admissible configurations inside each block. Assume for simplicity that
\(2k\mid n\), and partition the coordinates \([n]\) into \(B=n/(2k)\) disjoint
blocks \(B_1,\dots,B_B\), each of size \(2k\). Within a single block, let
\[
V
:=
\left\{
\mathbf0^{2k},
\mathbf1^{2k},
(\mathbf0^k,\mathbf1^k),
(\mathbf1^k,\mathbf0^k)
\right\}
\subseteq\{0,1\}^{2k}.
\]
Every two distinct elements of $V$ have Hamming distance either $k$ or
$2k$. The full truncation set is the Cartesian product
$S:=V^B$.
Equivalently, a vector \(x\in\{0,1\}^n\) lies in \(S\) if and only if, after
decomposing it into blocks $x=(x_{B_1},\dots,x_{B_B})$, every block belongs
to $V$. This product structure makes the conditional distribution on $S$
factor into the corresponding blockwise conditional laws on \(V\).
Moreover the Hamming separation inside \(V\) implies that no
single-bit flip, and more generally no small Hamming perturbation of radius
less than \(k\), can move between distinct admissible states in the same
block. 

\section*{Conclusion and Future Work}
In this paper, we present a novel framework for parameter estimation in truncated Boolean product distributions beyond assumptions on the survival mass of the domain or stringent local-connectivity assumptions. We give algorithms with improved sample complexities for learning under \emph{fatness} and extend this notion using influence to provide efficient guarantees beyond bounded width.

The present work opens the door to important future questions:
(i) What is the precise relationship between the influence and the Fisher information matrix used in the anti-concentration condition of \cite{fotakis2022efficientparameterestimationtruncated}?
(ii) As our framework does not require sampling from arbitrary parameterizations of the underlying model, do our techniques transfer to more complex discrete distributions where sampling is known to be difficult?
(iii) Lastly, is it possible to learn truncated Boolean product distributions when the truncation set is unknown but lies in a class with sufficiently low influences?

\newpage
\bibliographystyle{alpha}
\bibliography{bib/hard_constrained, bib/new-bibs, bib/truncated}
\newpage 
\appendix
\section{Further Analysis of \cite{fotakis2022efficientparameterestimationtruncated}\label{app:fotakis}}
In this section, we discuss the assumptions of \cite{fotakis2022efficientparameterestimationtruncated} in greater detail and explain how they relate to the width of the Boolean product distribution. To review, the authors make the following anti-concentration assumption on this measure.
\begin{assumption}[Assumptions 3 and 4 of \cite{fotakis2022efficientparameterestimationtruncated}]
\label{asmp:fotakis}
Assume the following is true,
    \begin{itemize}
    \item The mass of the truncation set $\mu_{\bz}(S) \ge \alpha \in \Theta(1)$.
        \item There exists a $\lambda \in \Theta(1)$ such that for all unit vectors $\bw \in \R^n$, $\|\bw\|_2 = 1$, and all $c \in \R$, $\pr_{\bx \sim \mu_{\bz\mid S}}[\bw^\top \bx \not \in (c-\lambda, c + \lambda) ] \ge \lambda$.
    \end{itemize}
\end{assumption}
This statement can be translated directly into a lower bound on the variance of linear forms under this measure.
\begin{lemma}
    Let $\mu_{\bz\mid S}$ satisfy Assumption~\ref{asmp:fotakis} with parameters $\alpha,\lambda$. It follows that for all $\bw\in\R^n$ with $\|\bw\|_2=1$, $\Var(\bw^\top \bx) \ge \lambda^3$. Moreover, for all $i \in [n]$, $\Var(\bx_i) \ge \lambda^3$.
\end{lemma}
\begin{proof}
  Fix $\bw\in\R^n$ with $\|\bw\|_2=1$, and set $c=\ex_{\bx\sim\mu_{\bz\mid S}}[\bw^\top\bx]$. The assumption, combined with Chebyshev's inequality, implies that
  \begin{equation*}
      \lambda \le \pr_{\bx \sim \mu_{\bz\mid S}}[|\bw^\top \bx - \ex_{\bx\sim \mu_{\bz\mid S}} [\bw^\top \bx]| \ge \lambda]
       \le \frac{\Var(\bw^\top \bx)}{\lambda^2}
  \end{equation*} Rearranging yields the desired result. 
  Next, taking $\bw=\be_i$, the $i$-th standard basis vector in $\R^n$, gives $\Var(\bx_i)\ge\lambda^3$ for every $i\in[n]$.
\end{proof}

Recall further the relationship between the natural parameters $\bz$ and the underlying probabilities of the model $\bp$. As $z_i = \log(p_i/(1-p_i))$ it follows that if $\bp$ is a constant with respect to $n$, so is $\bz$, and if $\bp \ge 1/\poly(n)$ then $\|\bz\|_\infty \in \calO(\log(n))$.

We also recall their associated lower bound. 
\begin{lemma}[Lemma 4 \cite{fotakis2022efficientparameterestimationtruncated}]
    Let $\mu_{\bp}$ be a Boolean product distribution and let $\mu_{\bp\mid S}$ be a truncation of $\mu_{\bp}$. Assume that
the anti-concentration parameter $\lambda^*$ satisfies $1/\lambda^* = \omega(\poly(n))$. Then, computing an estimation $\hat \bp$
of the parameter vector $\bp$ of $\mu_{\bp}$ such that $\|\bp - \hat\bp\|_2 \le o(1)$ requires an expected number of $\Omega(1/\lambda^*)$
samples from $\mu_{\bp\mid S}$.
\end{lemma}
We note that, given their approach, which uses projected gradient descent, it is infeasible to estimate beyond the barrier of $\lambda \in \Theta(1)$, as the strong convexity of the maximum likelihood objective in the ball around the parameters decays at a rate of $\alpha^{\poly(1/\lambda)}$.

\subsection{Universality of Influence}
We now give a short proof of the fact that, given any set whose affine span is $\R^n$, the influence sequence $\{\Inf_I^{\,\mu_{\bz\mid S}}\}_{I \subseteq [n]}$ is nonempty, and the weighted matrix found in Assumption~\ref{asmp:anti} is positive definite.
\begin{lemma}
\label{lem:affine-span-nonzero-influence-pd}
Let $S\subseteq\{0,1\}^n$ satisfy $\operatorname{aff}(S)=\mathbb R^n$, and assume $\mu_{\bz\mid S}(\bx)>0$ for every $\bx\in S$. For each $I\subseteq[n]$, define
$\alpha_I:=\Pr_{\bx\sim\mu_{\bz\mid S}}(\bx^{\oplus I}\in S)$ and
$\Inf_I^{\mu_{\bz\mid S}}:=1-\alpha_I$. Then there exist nonempty sets
$I_1,\dots,I_n\subseteq[n]$ such that $\alpha_{I_r}>0$, equivalently
$\Inf_{I_r}^{\mu_{\bz\mid S}}<1$, and signed vectors $w_r\in\{0,\pm1\}^n$ supported on $I_r$ such that $w_1,\dots,w_n$ are linearly independent. Consequently, the signed influence matrix
\[
M_S:=\sum_{\bx,\by\in S}\mu_{\bz\mid S}(\bx)\mu_{\bz\mid S}(\by)(\by-\bx)(\by-\bx)^\top
\]
is positive definite.
\end{lemma}

\begin{proof}
Since $\operatorname{aff}(S)=\mathbb R^n$, there exist points
$\bx^{(0)},\bx^{(1)},\dots,\bx^{(n)}\in S$ such that
$w_r:=\bx^{(r)}-\bx^{(0)}$, $r=1,\dots,n$, are linearly independent. For each
$r$, let $I_r:=\{i:x_i^{(r)}\neq x_i^{(0)}\}$. Then $I_r\neq\emptyset$ and
$\bx^{(r)}=(\bx^{(0)})^{\oplus I_r}$. Since $\mu_{\bz\mid S}(\bx^{(0)})>0$, we have
$\alpha_{I_r}\ge \mu_{\bz\mid S}(\bx^{(0)})>0$, and therefore
$\Inf_{I_r}^{\mu_{\bz\mid S}}<1$. Moreover, each $w_r=\bx^{(r)}-\bx^{(0)}$ is
supported on $I_r$, belongs to $\{0,\pm1\}^n$, and the vectors
$w_1,\dots,w_n$ are linearly independent by construction.

It remains to show that $M_S\succ0$. For any $v\in\mathbb R^n$,
$v^\top M_Sv=\sum_{\bx,by\in S}\mu_{\bz\mid S}(\bx)\mu_{\bz\mid S}(\by)
\langle v,\by-\bx\rangle^2\ge0$. If $v^\top M_Sv=0$, then
$\langle v,\by-\bx\rangle=0$ for every $\bx,by\in S$, since all weights are positive.
Thus $v$ is orthogonal to $\operatorname{span}\{\by-\bx:\bx,\by\in S\}$. Because
$\operatorname{aff}(S)=\mathbb R^n$, this span is all of $\mathbb R^n$, so
$v=0$. Hence $M_S$ is positive definite.
\end{proof}

\section{Omitted Proofs of Section \ref{sec:fat}}
\subsection{Proof of Theorem \ref{thm:fat_rel}}
In this section, we prove Theorem~\ref{thm:fat_rel}, which provides parameter-estimation guarantees for learning under the fatness condition and is restated below for reference.

\begin{theoremfatrel}
    Let $\bx^{(1)},\ldots,\bx^{(N)}$ be independent random variables drawn from a $\gamma$-fat distribution $\mu_{\bz\mid S}$, where $\|\bz\|_\infty\le R$. For every $\epsilon,\delta\in(0,1)$, if
    \[
    N\ge
    C\frac{n e^{2R}}{\gamma\epsilon^2}
    \log\left(\frac{2n}{\delta}\right),
    \]
    then there is an algorithm that returns an estimate $\hat{\bz}$ satisfying
    $\|\bz-\hat{\bz}\|_2\le\epsilon$ with probability at least $1-\delta$.
\end{theoremfatrel}

We first isolate the one-dimensional estimation problem on the selected
subsample.
\begin{lemma}
\label{lem:fat_rel}
Let $(F_j,X_j)_{j=1}^N$ be independent and identically distributed, where
$F_j\in\{0,1\}$ is observed,
$\Pr(F_j=1)\ge\gamma$, and
\[
X_j\mid\{F_j=1\}\sim\operatorname{Bernoulli}(\sigma(z))
\qquad\text{for some }|z|\le R.
\]
Let $M:=\sum_{j=1}^NF_j$, set $Y_j:=2X_j-1$, and, when $M>0$, define
\[
\widehat z
\in
\operatorname*{arg\,min}_{w\in[-R,R]}
\frac1M\sum_{j:F_j=1}
\log(1+e^{-Y_jw}).
\]
If $M=0$, set $\widehat z:=0$.
There is a universal constant $C>0$ such that, for every
$\epsilon,\delta\in(0,1)$,
\[
N
\ge
C\frac{e^{2R}}{\gamma\epsilon^2}
\log\left(\frac{4}{\delta}\right)
\]
implies $|\widehat z-z|\le\epsilon$ with probability at least $1-\delta$.
\end{lemma}

\begin{proofthmfatrel}
For each coordinate $i$, take
$F_j=\mathbf 1\{(\bx^{(j)})^{\oplus i}\in S\}$ and
$X_j=x_i^{(j)}$. The fatness assumption gives
$\Pr(F_j=1)\ge\gamma$. Moreover, the event $F_j=1$ depends only on
$\bx_{-i}^{(j)}$: it holds precisely when both completions
$(0,\bx_{-i}^{(j)})$ and $(1,\bx_{-i}^{(j)})$ belong to $S$.
The product structure therefore implies
\[
x_i^{(j)}\mid\{F_j=1\}
\sim
\operatorname{Bernoulli}(p_i)
=
\operatorname{Bernoulli}(\sigma(z_i)).
\]
Thus Lemma~\ref{lem:fat_rel} applies to every coordinate.

Apply the lemma with target accuracy $\epsilon/\sqrt n$ and failure
probability $\delta/n$. A union bound gives
\[
\Pr\left(
\max_{i\in[n]}|\widehat z_i-z_i|>\frac{\epsilon}{\sqrt n}
\right)
\le
\sum_{i=1}^n
\Pr\left(
|\widehat z_i-z_i|>\frac{\epsilon}{\sqrt n}
\right)
\le\delta.
\]
On the complementary event,
$\|\widehat{\bz}-\bz\|_2\le
\sqrt n\|\widehat{\bz}-\bz\|_\infty\le\epsilon$.
The asserted sample bound is exactly the resulting requirement from
Lemma~\ref{lem:fat_rel}.
\end{proofthmfatrel}

\subsubsection{Proof of Lemma \ref{lem:fat_rel}}
\begin{prooflemfatrel}
Let
\[
\ell_y(w):=\log(1+e^{-yw}),
\qquad
\calL(w):=\E[\ell_Y(w)\mid F=1].
\]
Since $Y\mid\{F=1\}$ has logistic natural parameter $z$,
\[
\calL'(z)=0,
\qquad
\calL''(w)
=
\sigma(w)(1-\sigma(w))
\ge
\frac{e^R}{(1+e^R)^2}
=:\kappa_R
\quad
\text{for all }w\in[-R,R].
\]
In particular, $\kappa_R\ge e^{-R}/4$. Whenever $z+\epsilon\le R$,
strong convexity gives $\calL'(z+\epsilon)\ge\kappa_R\epsilon$;
whenever $z-\epsilon\ge-R$, it gives
$\calL'(z-\epsilon)\le-\kappa_R\epsilon$.

Condition on the selected index set and on $M=m\ge1$. The selected
variables are independent draws from the conditional Bernoulli law.
Furthermore, $\ell_Y'(w)=-Y\sigma(-Yw)\in[-1,1]$. Hoeffding's
inequality therefore yields
\[
\Pr\left(
\widehat{\calL}'(z+\epsilon)\le0
\mid M=m
\right)
\le
\exp\left(-\frac{m\kappa_R^2\epsilon^2}{2}\right)
\]
whenever the upper endpoint lies in $[-R,R]$, and the same bound holds
for
$\Pr(\widehat{\calL}'(z-\epsilon)\ge0\mid M=m)$.
If an endpoint lies outside $[-R,R]$, the corresponding event
$\widehat z>z+\epsilon$ or $\widehat z<z-\epsilon$ is impossible.
By convexity, the two gradient bounds thus imply
$|\widehat z-z|\le\epsilon$.

Writing $q:=\Pr(F=1)\ge\gamma$, we have
$M\sim\operatorname{Bin}(N,q)$, and the multiplicative Chernoff bound
gives
\[
\Pr\left(M<\frac{\gamma N}{2}\right)
\le
\Pr\left(M<\frac{qN}{2}\right)
\le
\exp\left(-\frac{\gamma N}{8}\right).
\]
Consequently,
\[
\Pr(|\widehat z-z|>\epsilon)
\le
\exp\left(-\frac{\gamma N}{8}\right)
+
2\exp\left(-\frac{\gamma N\kappa_R^2\epsilon^2}{4}\right).
\]
Using $\kappa_R^2\ge e^{-2R}/16$ shows that the sample size in the
lemma makes the right-hand side at most $\delta$, after increasing the
universal constant $C$.
\end{prooflemfatrel}
\subsubsection{Bisection Algorithm for Learning in Relative Error}
Recall that in the previous section, we proved that the empirical loss $\hat \calL_i$ is convex for every $i\in[n]$.
For each coordinate \(i\in[n]\), the empirical objective is a one-dimensional convex function. It is therefore unnecessary to run a general high-dimensional convex optimization method. Instead, we minimize each coordinate-wise empirical loss by bisection on its derivative. Let \(y_i^{(j)}:=2x_i^{(j)}-1\in\{-1,1\}\), and let $\textsf{FLIP}(\bx^{(j)}, i)$ denote whether the \(j\)-th sample is flippable in coordinate \(i\). Our goal is then to minimize 
\[
    \nabla \hat{\calL}_i(w)
    =
    -\frac{1}{M_i}
    \sum_{j\in\calJ_i}
    y_i^{(j)}
    \sigma(-y_i^{(j)}w),
\]
where \(\sigma(t)=1/(1+e^{-t})\). 

\begin{algorithm}[H]
\caption{Coordinate-wise minimization of the empirical logistic loss}
\label{alg:coordinate-loss-minimization}
\begin{algorithmic}[1]
\Require Samples \(\bx^{(1)},\dots,\bx^{(N)}\sim\mu_{\bz\mid S}\), truncation set \(S\), radius \(R\), optimization tolerance \(\tau>0\)
\Ensure Estimate \(\hat{\bz}\in[-R,R]^n\)

\For{\(i=1,\dots,n\)}
    \State \(\calJ_i\gets\{j\in[N]:(\bx^{(j)})^{\oplus i}\in S\}\), \(M_i\gets|\calJ_i|\)
    \If{\(M_i=0\)}
        \State \(\hat z_i\gets0\)
    \Else
    \State Define \(g_i(w):=\nabla \hat{\calL}_i(w)\)
    \If{\(g_i(-R)\ge 0\)}
        \State \(\hat z_i\gets -R\)
    \ElsIf{\(g_i(R)\le 0\)}
        \State \(\hat z_i\gets R\)
    \Else
        \State \(a\gets -R\), \(b\gets R\)
        \While{\(b-a>\tau\)}
            \State \(m\gets (a+b)/2\)
            \If{\(g_i(m)\le 0\)}
                \State \(a\gets m\)
            \Else
                \State \(b\gets m\)
            \EndIf
        \EndWhile
        \State \(\hat z_i\gets (a+b)/2\)
    \EndIf
    \EndIf
\EndFor

\State \Return \(\hat{\bz}=(\hat z_1,\dots,\hat z_n)\)
\end{algorithmic}
\end{algorithm}
\emph{Remarks:} For a fixed coordinate \(i\), evaluating \(g_i(w)\) requires summing over the \(N\) samples, and therefore costs \(\calO(N)\) arithmetic operations. Bisection over an interval of length \(2R\) to accuracy \(\tau\) requires $\calO(\log(R/\tau))$ operations \cite{BV04}. Therefore, running the bisection over all coordinates yields a total runtime of $    \calO\left(
        nN\log\frac{R}{\tau}
    \right).$
The statistical proof shows that the exact empirical minimizer \(\widetilde z_i\) lies within \(\epsilon\) of \(z_i\) with high probability. The algorithm returns \(\hat z_i\) satisfying \(|\hat z_i-\widetilde z_i|\le \tau\). Therefore,
\[
    |\hat z_i-z_i|
    \le
    |\hat z_i-\widetilde z_i|
    +
    |\widetilde z_i-z_i|
    \le
    \tau+\epsilon.
\]
Thus choosing, for example, \(\tau=\epsilon\) only changes constants in the final estimation guarantee. For an \(\ell_2\) guarantee, one may instead take \(\tau=\epsilon/\sqrt n\), giving total runtime $    \calO\left(
        nN\log\frac{R\sqrt n}{\epsilon}
    \right).$

\subsection{Proof of Theorem \ref{thm:fat_abs}}
In this brief section, we prove Theorem~\ref{thm:fat_abs}. Our estimator $\hat \bp$ is constructed coordinate by coordinate by taking the empirical average of all samples that are flippable in coordinate $i$.

\begin{algorithm}[H]
\caption{Learning Bernoulli Probabilities from Flippable Samples}
\label{alg:learn-bernoulli-flippable}
\begin{algorithmic}[1]
\Require Samples \(\bx^{(1)},\dots,\bx^{(N)} \sim \mu_{\bp\mid S}\), truncation set \(S\subseteq \{0,1\}^n\)
\Ensure Estimate \(\hat{\bp}\in[0,1]^n\)

\For{\(i=1,\dots,n\)}
    \State \(N_i\gets0\), \(S_i\gets0\)
    \For{\(j=1,\dots,N\)}
        \If{$\textsf{FLIP}(\bx^{(j)}, i)$}
            \State \(N_i \gets N_i+1\) \Comment{Number of samples flippable in $i$}
            \State \(S_i \gets S_i + x_i^{(j)}\) \Comment{Running sum of flippable samples}
        \EndIf
    \EndFor
    \If{\(N_i=0\)}
        \State \(\hat p_i\gets1/2\)
    \Else
        \State \(\hat p_i \gets S_i/N_i\)
    \EndIf

\EndFor

\State \Return \(\hat{\bp}=(\hat p_1,\dots,\hat p_n)\)
\end{algorithmic}
\end{algorithm}
\begin{proofthmfatabs}
For each coordinate $i\in[n]$, let
\[
q_i
:=
\Pr_{\bx\sim\mu_{\bp\mid S}}(\bx^{\oplus i}\in S)
\ge\gamma,
\]
and let $N_i$ be the number of samples flippable in coordinate $i$.
Then $N_i\sim\operatorname{Bin}(N,q_i)$, so the multiplicative Chernoff
bound gives
\[
\Pr\left(N_i<\frac{\gamma N}{2}\right)
\le
\Pr\left(N_i<\frac{q_iN}{2}\right)
\le
\exp\left(-\frac{\gamma N}{8}\right).
\]
Whenever $N_i>0$, define
\[
\hat p_i
:=
\frac{1}{N_i}
\sum_{\substack{j\in[N]:\\(\bx^{(j)})^{\oplus i}\in S}}
x_i^{(j)}.
\]
As in the proof of Theorem~\ref{thm:fat_rel}, conditional on
flippability the selected coordinates are independent
$\operatorname{Bernoulli}(p_i)$ variables. Thus, conditional on
$N_i=m\ge\gamma N/2$, Hoeffding's inequality gives
\[
\Pr\left(|\hat p_i-p_i|>\epsilon\mid N_i=m\right)
\le
2e^{-2m\epsilon^2}
\le
2e^{-\gamma N\epsilon^2}.
\]
Combining the two bounds and taking a union bound over $i\in[n]$ yields
\[
\Pr\left(\|\hat{\bp}-\bp\|_\infty>\epsilon\right)
\le
n e^{-\gamma N/8}
+
2n e^{-\gamma N\epsilon^2}.
\]
For $\epsilon\in(0,1)$, the right-hand side is at most $\delta$ when
\[
N
\ge
C\frac{1}{\gamma\epsilon^2}
\log\left(\frac{2n}{\delta}\right)
\]
for a sufficiently large universal constant $C$. This proves the
$\ell_\infty$ statement. Applying it with target accuracy
$\epsilon/\sqrt n$ gives
$\|\hat{\bp}-\bp\|_2\le\epsilon$ and the asserted $\ell_2$ sample
complexity.
\end{proofthmfatabs}

\section{Omitted Proofs of Section \ref{sec:influence}}
\label{sec:minimal-flip-estimation}

Before proving Theorem~\ref{thm:influence}, we introduce some notation.

For every nonempty $I\subseteq[n]$ with $|I|\le k$ and every
$s\in\{\pm1\}^I$ satisfying $s_{\min I}=+1$, we call $(I,s)$ a
\emph{candidate pair}.  Thus, for each such $I$, the candidates contain
exactly one representative of every antipodal pair $\{s,-s\}$.  For a
candidate $(I,s)$, let $\bw_{I,s}\in\{0,\pm1\}^n$ be the extension
of $s$ by zero outside $I$, and define the paired minimal-flip set
\begin{equation*}
E_{I,s}
:=
\left\{\bx\in S:
\begin{array}{l}
\bs_I(\bx)\in\{s,-s\},\quad
\bx^{\oplus I}\in S,\\[1mm]
\bx^{\oplus J}\notin S
\quad\text{for every }\varnothing\neq J\subsetneq I
\end{array}
\right\},
\label{eq:paired-event}
\end{equation*}
and its probability
\[
p_{I,s}:=\Pr(\bx\in E_{I,s}).
\]

\begin{theoreminf}
Fix $k$, and let $\bx^{(1)},\ldots,\bx^{(N)}$ be independent random variables drawn from $\mu_{\bz\mid S}$, where $\mu_{\bz\mid S}$ satisfies Assumptions~\ref{assmpt:minimal} and~\ref{asmp:anti}, and $\|\bz\|_\infty\le R$. There is a universal constant $C>0$ such that, for every $\epsilon,\delta\in(0,1)$, if
\[
N
\ge
C\frac{2^k e^{2kR}}{\gamma\lambda_s\epsilon^2}
\left(
k\log(2n)+\log\frac{4k}{\delta}
\right),
\]
then a polynomial-time algorithm returns an estimate $\widehat{\bz}$ satisfying
$\|\widehat{\bz}-\bz\|_2\le\epsilon$ with probability at least $1-\delta$.
\end{theoreminf}

The proof is organized around two main lemmas, (i) ensuring there are sufficiently many flip directions to learn the underlying parameters, and (ii) ensuring that the linear system formed by these directions is well-conditioned.

The final subsection first verifies \eqref{eq:informative-design}
directly from Assumptions~\ref{assmpt:minimal} and~\ref{asmp:anti},
and then proves Theorem~\ref{thm:influence} by combining these results.
\subsection{The Algorithm}
 The estimator never needs to know or sample from the hidden support
family.  It works only with candidate pairs
observed in the data and computes well-conditioned design weights over
those pairs.
For the remainder of this section, set
\[
p_0:=\frac{\gamma}{2^k}.
\]
The analysis uses the following informative-design condition: there is
an unknown distribution $\pi_\star$ over the candidate pairs $(I,s)$
such that
\begin{equation*}
p_{I,s}\ge p_0
\quad\text{for every }(I,s)\in\operatorname{supp}(\pi_\star),
\qquad
\E_{(I,s)\sim\pi_\star}
\!\left[
\bw_{I,s}\bw_{I,s}^{\mathsf T}
\right]
\succeq
\lambda_s I_n.
\label{eq:informative-design}
\end{equation*}
The inputs $k,R,p_0$, and $\lambda_s$ are known to the learner: $R$ is a
valid upper bound on $\|\bz\|_\infty$, while $p_0$ and $\lambda_s$
are valid lower bounds in \eqref{eq:informative-design}.  The witness
$\pi_\star$ is not known.

Given samples
$\bx^{(1)},\ldots,\bx^{(N)}$, we estimate each paired-event mass
by
\[
\widehat p_{I,s}
:=
\frac1N
\sum_{\ell=1}^N
\mathbf 1\{\bx^{(\ell)}\in E_{I,s}\}.
\]
For each support $I$, a sample can contribute to only one canonical sign
pair.  Thus the implementation need only materialize pairs that actually
appear; every unobserved pair has $\widehat p_{I,s}=0$.
Retain
\begin{equation*}
\widehat{\calR}
:=
\left\{
(I,s):
\widehat p_{I,s}\ge\frac{3p_0}{4}
\right\}.
\label{eq:retained-set}
\end{equation*}
For every candidate $(I,s)$ that occurs at least once in the sample,
define $\widehat y_{I,s}$ to be the constrained maximum-likelihood
estimator
\begin{equation*}
\widehat y_{I,s}
\in
\operatorname*{arg\,min}_{|u|\le |I|R}
\sum_{\ell:\,\bx^{(\ell)}\in E_{I,s}}
\left[
\log\left(2\cosh\frac{u}{2}\right)
-
\frac{u}{2}\left(
2\mathbf 1\{\bs_I(\bx^{(\ell)})=s\}-1
\right)
\right].
\label{eq:form-estimator}
\end{equation*}
If the pair never occurs, set $\widehat y_{I,s}:=0$.  The algorithm
computes and uses this quantity only for retained pairs; the threshold in
\eqref{eq:retained-set} ensures that their defining sums are nonempty.

Next, find weights $q$ on $\widehat{\calR}$ satisfying
\begin{equation*}
\begin{aligned}
q_{I,s}&\ge0
&&\text{for every }(I,s)\in\widehat{\mathcal R},\\
\sum_{(I,s)\in\widehat{\calR}}q_{I,s}&=1,
&\qquad
\sum_{(I,s)\in\widehat{\calR}}
q_{I,s}
\bw_{I,s}\bw_{I,s}^{\mathsf T}
&\succeq
\frac{\lambda_s}{2}I_n.
\end{aligned}
\label{eq:observable-design}
\end{equation*}
We then solve the following minimization problem.
\begin{equation*}
\widehat{\bz}
\in
\operatorname*{arg\,min}_{\bu\in\mathbb R^n}
\sum_{(I,s)\in\widehat{\calR}}
q_{I,s}
\left(
\widehat y_{I,s}
-
\langle \bw_{I,s},\bu\rangle
\right)^2.
\label{eq:weighted-least-squares}
\end{equation*}

\subsection{Finding Directions}
In this section, we wish to demonstrate that given Assumptions \ref{assmpt:minimal}, \ref{asmp:anti}, we can find sufficiently many well conditioned directions to learn the underlying parameters. 
We note that there are at most
\[
\sum_{r=1}^k2^{r-1}\binom nr
\le
k(2n)^k
\]
candidate pairs.  This elementary count is the only combinatorial factor
in the concentration arguments below. Our primary result in this section is demonstrating the following sample complexity bound on estimating linear forms.

\begin{lemma}[Uniform data event]
\label{lem:uniform-data-event}
There is a universal constant $C>0$ such that, if
\begin{equation*}
N
\ge
C
\left(
\frac1{p_0}
+
\frac{e^{2kR}}{p_0\eta^2}
\right)
\left(
k\log(2n)+\log\frac{4k}{\delta}
\right),
\label{eq:uniform-data-bound}
\end{equation*}
then, with probability at least $1-\delta$, all three statements below
hold simultaneously:
\begin{enumerate}
    \item
    $\operatorname{supp}(\pi_\star)
    \subseteq\widehat{\mathcal R}$;
    \item
    every $(I,s)\in\widehat{\mathcal R}$ satisfies
    $p_{I,s}>p_0/2$;
    \item
    every $(I,s)\in\widehat{\mathcal R}$ satisfies
    \[
    \left|
    \widehat y_{I,s}
    -
    \langle \bw_{I,s},\bz\rangle
    \right|
    \le\eta.
    \]
\end{enumerate}
\end{lemma}
Towards demonstrating the above lemma, we first show that the linear form $\langle \bw_{I,s},\bz\rangle$ can be estimated with high probability from the samples in $E_{I,s}$.
\begin{lemma}[Estimation of one linear form]
\label{lem:fixed-form-estimation}
For every candidate $(I,s)$ with $p_{I,s}>0$, conditional on
$\bx\in E_{I,s}$, the orientation
\[
A
:=
\begin{cases}
+1,&\bs_I(\bx)=s,\\
-1,&\bs_I(\bx)=-s
\end{cases}
\]
has distribution
\[
\Pr_u(A=a)
=
\frac{e^{au/2}}{2\cosh(u/2)},
\qquad
a\in\{\pm1\},
\qquad
u:=\langle \bw_{I,s},\bz\rangle.
\]
Moreover, for every candidate $(I,s)$ and every $\eta>0$,
\begin{equation*}
\Pr\left(
\left|
\widehat y_{I,s}
-
\langle \bw_{I,s},\bz\rangle
\right|>\eta
\right)
\le
\exp\left(-\frac{Np_{I,s}}8\right)
+
2\exp\left(
-\frac{Np_{I,s}\eta^2}{16e^{|I|R}}
\right).
\label{eq:fixed-form-tail}
\end{equation*}
\end{lemma}

\subsubsection{Proof of Lemma~\ref{lem:uniform-data-event}}
\begin{proof}
Fix a candidate with $p_{I,s}\ge p_0$.  A multiplicative Chernoff bound
gives
\[
\Pr\left(
\widehat p_{I,s}<\frac{3p_0}{4}
\right)
\le
\Pr\left(
\widehat p_{I,s}<\frac{3p_{I,s}}{4}
\right)
\le
\exp\left(-\frac{Np_0}{32}\right).
\]
Conversely, if $p_{I,s}\le p_0/2$, then stochastic monotonicity of the
binomial distribution and the upper-tail Chernoff bound give
\[
\begin{aligned}
\Pr\left(
\widehat p_{I,s}\ge\frac{3p_0}{4}
\right)
&\le
\Pr\left(
\frac1N\operatorname{Bin}\left(N,\frac{p_0}{2}\right)
\ge
\frac{3p_0}{4}
\right)\\
&\le
\exp\left(-\frac{Np_0}{20}\right)
\le
\exp\left(-\frac{Np_0}{32}\right).
\end{aligned}
\]
There are at most $k(2n)^k$ candidate pairs.  Therefore a union bound
shows that the probability that either the first or second conclusion
of the lemma fails is at most
\[
2k(2n)^k\exp\left(-\frac{Np_0}{32}\right).
\]

For the third conclusion, apply
Lemma~\ref{lem:fixed-form-estimation} to every candidate with
$p_{I,s}\ge p_0/2$.  A second union bound shows that the probability
that any such estimate has error greater than $\eta$ is at most
\[
k(2n)^k
\left[
\exp\left(-\frac{Np_0}{16}\right)
+
2\exp\left(
-\frac{Np_0\eta^2}{32e^{kR}}
\right)
\right].
\]
The sample-size condition \eqref{eq:uniform-data-bound}, for a
sufficiently large universal constant $C$, makes the sum of the two
displayed failure probabilities at most $\delta$.

The estimation event is uniform over the deterministic collection of
all candidate pairs with $p_{I,s}\ge p_0/2$.  On the screening event,
every data-dependent retained pair belongs to this collection.
Consequently, the same samples may be used for screening and estimation;
no sample splitting or independence between the two events is required.
\end{proof}

\subsubsection{Proof of Lemma~\ref{lem:fixed-form-estimation}}
\begin{proof}
We first establish the asserted conditional law.  The map
$T_I(\bx):=\bx^{\oplus I}$ is a bijection between the points in
$E_{I,s}$ satisfying $\bs_I(\bx)=s$ and those satisfying
$\bs_I(\bx)=-s$.  Indeed, if $\mathbf y=T_I(\bx)$, then
$\mathbf y\in S$ and $\mathbf y^{\oplus I}=\bx\in S$.  Moreover, for
every $\varnothing\neq J\subsetneq I$,
\[
\mathbf y^{\oplus J}
=
\bx^{\oplus(I\setminus J)}
\notin S,
\]
because $I\setminus J$ is again a nonempty proper subset of $I$.  Thus
$T_I$ preserves the defining conditions of $E_{I,s}$ and is an
involution.

For every matched pair $\bx,T_I(\bx)$ with $\bs_I(\bx)=s$,
\[
\frac{
\mu_{\bz\mid S}(\bx)
}{
\mu_{\bz\mid S}(T_I(\bx))
}
=
\exp\left(
\langle\bz,\bx-T_I(\bx)\rangle
\right)
=
\exp\left(
\sum_{i\in I}s_i z_i
\right)
=
\exp\left(
\langle \bw_{I,s},\bz\rangle
\right).
\]
Summing over the points with $\bs_I(\bx)=s$ and normalizing by the
total mass of $E_{I,s}$ gives the stated logistic law.

Let $M$ be the number of observations falling in $E_{I,s}$.  If
$p_{I,s}=0$, then \eqref{eq:fixed-form-tail} is immediate, so assume
$p_{I,s}>0$.  Then
$M\sim\operatorname{Bin}(N,p_{I,s})$.  Conditional on the set of indices
counted by $M$, the corresponding observations are independent draws
from the law of $\bx$ given $\bx\in E_{I,s}$.  When $M=0$,
the conditional bound below is trivial; hence fix $M=m\ge1$ and associate
to each of these $m$ observations the orientation $A=+1$ when
    $\bs_I(\bx)=s$ and $A=-1$ when $\bs_I(\bx)=-s$.  By the first part
of the lemma, these orientations are i.i.d. from $\Pr_u$, where
$u:=\langle \bw_{I,s},\bz\rangle$.
Moreover, $|u|\le |I|R$, and \eqref{eq:form-estimator} is the
maximum-likelihood estimator of $u$ constrained to this interval.

As a function of the log-odds parameter, the log-partition function is
$\psi(v):=\log(2\cosh(v/2))$, and
\[
\psi''(v)
=
\frac14\operatorname{sech}^2(v/2)
\ge
\frac{1}{4e^{|I|R}}
\qquad
\text{for every }|v|\le |I|R.
\]
If $u+\eta\le |I|R$, the event
$\widehat y_{I,s}\ge u+\eta$ implies that the empirical mean of the
orientations is at least $\tanh((u+\eta)/2)$.  The exponential-family
Chernoff bound therefore gives
\[
\Pr\left(
\widehat y_{I,s}\ge u+\eta
\mid M=m
\right)
\le
\exp\left(
-mD(\Pr_{u+\eta}\|\Pr_u)
\right).
\]
Writing $D(\cdot\|\cdot)$ for Kullback--Leibler divergence, for this
family
\[
D(\Pr_v\|\Pr_u)
=
\psi(u)-\psi(v)-\psi'(v)(u-v).
\]
The displayed lower bound on $\psi''$ implies
\[
D(\Pr_{u+\eta}\|\Pr_u)
\ge
\frac{\eta^2}{8e^{|I|R}}.
\]
If $u+\eta>|I|R$, the upper-tail event is empty because the estimator is
constrained.  Applying the same argument to the lower tail gives
\[
\Pr\left(
\left|
\widehat y_{I,s}-u
\right|>\eta
\mid M=m
\right)
\le
2\exp\left(
-\frac{m\eta^2}{8e^{|I|R}}
\right).
\]
Finally, the multiplicative Chernoff bound gives
\[
\Pr\left(M<\frac{Np_{I,s}}2\right)
\le
\exp\left(-\frac{Np_{I,s}}8\right).
\]
On the complementary event, substitute
$m\ge Np_{I,s}/2$ into the preceding conditional bound to obtain
\eqref{eq:fixed-form-tail}.
\end{proof}

\subsection{Least-Squares Recovery}
We now demonstrate that given the inputs of the previous section, the linear system formed by the retained candidate pairs is well-conditioned and can be solved to recover the underlying parameters.  The following two lemmas formalize this argument.
\begin{lemma}[Designed recovery]
\label{lem:designed-recovery}
On the event of Lemma~\ref{lem:uniform-data-event}, the feasibility
problem \eqref{eq:observable-design} has a solution and the estimator
\eqref{eq:weighted-least-squares} satisfies
\[
\|\widehat{\bz}-\bz\|_2
\le
\eta\sqrt{\frac{2}{\lambda_s}}.
\]
\end{lemma}
Towards proving this lemma, we first show that the weighted least-squares problem is stable to perturbations in the linear forms.
\begin{lemma}[Weighted least-squares stability]
\label{lem:weighted-ls-stability}
Suppose $q$ satisfies
\eqref{eq:observable-design} and
\[
\left|
\widehat y_{I,s}
-
\langle \bw_{I,s},\bz\rangle
\right|
\le\eta
\qquad
\text{whenever }q_{I,s}>0.
\]
Then the minimizer in \eqref{eq:weighted-least-squares} is unique and
\[
\|\widehat{\bz}-\bz\|_2
\le
\eta\sqrt{\frac{2}{\lambda_s}}.
\]
\end{lemma}

\subsubsection{Proof of Lemma~\ref{lem:designed-recovery}}
\begin{proof}
On the event of Lemma~\ref{lem:uniform-data-event},
$\operatorname{supp}(\pi_\star)\subseteq\widehat{\calR}$.  Extend
$\pi_\star$ to $\widehat{\calR}$ by assigning zero mass to every
retained pair outside its support.  By \eqref{eq:informative-design},
\[
\sum_{(I,s)\in\widehat{\calR}}
\pi_\star(I,s)
\bw_{I,s}\bw_{I,s}^{\mathsf T}
\succeq
\lambda_s I_n.
\]
Thus the feasibility problem \eqref{eq:observable-design} has a
solution; indeed, $\pi_\star$ is feasible even with $\lambda_s I_n$ in
place of $(\lambda_s/2)I_n$.

The third conclusion of Lemma~\ref{lem:uniform-data-event} controls
every retained row and therefore every row receiving positive
$q$-mass.  Lemma~\ref{lem:weighted-ls-stability} now gives
\[
\|\widehat{\bz}-\bz\|_2
\le
\eta\sqrt{\frac{2}{\lambda_s}}.
\]
\end{proof}

\subsubsection{Proof of Lemma~\ref{lem:weighted-ls-stability}}
\begin{proof}
For this proof only, define
\[
G
:=
\sum_{(I,s)\in\widehat{\calR}}
q_{I,s}
\bw_{I,s}\bw_{I,s}^{\mathsf T}.
\]
By \eqref{eq:observable-design},
$G\succeq(\lambda_s/2)I_n$, so the weighted least-squares objective is
strictly convex and has a unique minimizer.

Let $A$ be the matrix whose row indexed by $(I,s)$ is
$\sqrt{q_{I,s}}\bw_{I,s}^{\mathsf T}$, omitting rows of
zero weight.  Let $\mathbf e$ be the vector whose corresponding entry is
\[
\sqrt{q_{I,s}}
\left(
\widehat y_{I,s}
-
\langle \bw_{I,s},\bz\rangle
\right).
\]
Then $A^{\mathsf T}A=G$ and
\[
\|\mathbf e\|_2^2
\le
\sum_{(I,s)\in\widehat{\calR}}
q_{I,s}\eta^2
=
\eta^2.
\]
The normal equation*s give
\[
\widehat{\bz}-\bz
=
(A^{\mathsf T}A)^{-1}A^{\mathsf T}\mathbf e.
\]
Since
$\sigma_{\min}(A)\ge\sqrt{\lambda_s/2}$, the norm of the left
pseudoinverse is at most $\sqrt{2/\lambda_s}$.  Therefore
\[
\|\widehat{\bz}-\bz\|_2
\le
\eta\sqrt{\frac{2}{\lambda_s}}.
\]
\end{proof}

\subsection{Proof of Theorem~\ref{thm:influence}}

\begin{proofthminf}
For every $I\in\calF$, Assumption~\ref{assmpt:minimal} gives
\[
\Pr(\bx^{\oplus I}\in S)\ge2\gamma
\quad\text{and}\quad
\sum_{\varnothing\neq J\subsetneq I}
\Pr(\bx^{\oplus J}\in S)
\le\gamma,
\]
where the sum is empty if $|I|=1$.  The union bound therefore gives
\[
\Pr\left(
\bx^{\oplus I}\in S,\ 
\bx^{\oplus J}\notin S
\text{ for every }\varnothing\neq J\subsetneq I
\right)
\ge\gamma.
\]
The events obtained by specifying $\bs_I(\bx)=t$ partition this
minimal-flip event.  Hence, for every $I\in\calF$, there is a sign
pattern $t_I\in\{\pm1\}^I$ such that
\[
\Pr\left(
\bs_I(\bx)=t_I,\ 
\bx^{\oplus I}\in S,\ 
\bx^{\oplus J}\notin S
\text{ for every }\varnothing\neq J\subsetneq I
\right)
\ge
\frac{\gamma}{2^{|I|}}.
\]
By the definition of the canonical maximizer $s_I$ preceding
Assumption~\ref{asmp:anti}, the paired event therefore satisfies
\[
p_{I,s_I}
\ge
\frac{\gamma}{2^{|I|}}
\ge
\frac{\gamma}{2^k}
=p_0.
\]
Let $\pi_\star$ be the uniform distribution on
$\{(I,s_I):I\in\calF\}$.  Assumption~\ref{asmp:anti} gives
\[
\E_{(I,s)\sim\pi_\star}
\left[
\bw_{I,s}\bw_{I,s}^{\mathsf T}
\right]
=
\frac1{|\calF|}
\sum_{I\in\calF}
\bw_{I,s_I}\bw_{I,s_I}^{\mathsf T}
\succeq
\lambda_s I_n.
\]
Thus \eqref{eq:informative-design} holds with
$p_0=\gamma/2^k$ and design parameter $\lambda_s$.

Set
\[
\eta:=\epsilon\sqrt{\frac{\lambda_s}{2}}.
\]
Under the event of Lemma~\ref{lem:uniform-data-event},
Lemma~\ref{lem:designed-recovery} gives
\[
\|\widehat{\bz}-\bz\|_2
\le
\eta\sqrt{\frac{2}{\lambda_s}}
=
\epsilon.
\]

It remains to verify the sample bound.  Since
$\eta^2=\epsilon^2\lambda_s/2$,
\[
\frac{e^{2kR}}{p_0\eta^2}
=
\frac{2e^{2kR}}
     {p_0\lambda_s\epsilon^2}.
\]
Moreover, taking traces in \eqref{eq:informative-design} gives
\[
n\lambda_s
\le
\E_{(I,s)\sim\pi_\star}
\|\bw_{I,s}\|_2^2
=
\E_{(I,s)\sim\pi_\star}|I|
\le
k
\le n.
\]
Thus $\lambda_s\le1$.  Because also $\epsilon\le1$ and
$e^{2kR}\ge1$, the sample-size condition in
Theorem~\ref{thm:influence}, for a sufficiently large universal
constant $C$, implies \eqref{eq:uniform-data-bound}. The statistical
claim follows from Lemma~\ref{lem:uniform-data-event} and
Lemma~\ref{lem:designed-recovery}.

Finally, for each sample the learner queries
$\mathbf 1\{(\bx^{(\ell)})^{\oplus J}\in S\}$ once for every
nonempty $J$ with $|J|\le k$ and caches the answers.  This uses exactly
\[
N\sum_{r=1}^k\binom nr
\]
membership queries.  The screening and likelihood calculations have
size $Nn^{O(k)}$.  The feasibility problem
\eqref{eq:observable-design} has at most $k(2n)^k$ scalar weights and
one $n\times n$ semidefinite constraint.  The hidden design satisfies
the stronger inequality with $\lambda_s I_n$, giving slack
$\lambda_s/2$ relative to \eqref{eq:observable-design}; hence a standard
polynomial-time semidefinite-programming method can return a certified
feasible solution to the required accuracy.  Weighted least squares is
polynomial time as well.

\end{proofthminf}

\section{Omitted Proofs of Section \ref{sec:lower}}
In this section, we prove an information-theoretic lower bound for estimation under sparse truncation.
To this end, we introduce Assouad's lemma, a key tool used extensively in statistical decision theory to prove minimax lower bounds for estimation.
\begin{lemma}[Assouad's Lemma \cite{yu1997assouad}]
Let \(\{P_\omega:\omega\in\{0,1\}^B\}\) be a family of distributions indexed by the binary hypercube, and let \(\{\theta^\omega:\omega\in\{0,1\}^B\}\subseteq\Theta\) be the corresponding parameters. Suppose that, for some \(\rho>0\) and a pseudo-metric $d$ on $\Theta$,
\[
d^2(\theta^\omega,\theta^{\omega'})
\ge
\rho^2 d_H(\omega,\omega')
\qquad
\text{for all } \omega,\omega'\in\{0,1\}^B,
\]
where \(d_H\) denotes Hamming distance. Then, for any estimator \(\widehat\theta\),
\[
\sup_{\omega\in\{0,1\}^B}
\mathbb E_\omega d^2(\widehat\theta,\theta^\omega)
\ge
\frac{\rho^2}{8}
\sum_{b=1}^B
\left(
1-
\operatorname{TV}\!\left(P_{\omega:b=0},P_{\omega:b=1}\right)
\right),
\]
where \(P_{\omega:b=a}\) denotes the mixture distribution obtained by drawing
\(\omega\) uniformly over \(\{0,1\}^B\) conditional on \(\omega_b=a\).
\end{lemma}

\subsection{Proof of Theorem \ref{thm:lower_bound}}
\begin{proofthmlower}
Let $v_1:=\mathbf0^{2k}\in V$ and
$r:=2v_1-\mathbf1=-\mathbf1\in\{\pm1\}^{2k}$. Set
\[
\eta:=\frac{8\epsilon}{\sqrt n}.
\]
The assumed range $\epsilon\le R\sqrt n/16$ ensures $\eta\le R/2$.
For each block $b$, define
\[
\bz_b^{(0)}:=(R-\eta)r,
\qquad
\bz_b^{(1)}:=Rr,
\qquad
\Delta:=\bz_b^{(1)}-\bz_b^{(0)}=\eta r.
\]
Both parameters belong to $[-R,R]^{2k}$, and
\[
\|\Delta\|_2^2=2k\eta^2.
\]

We first bound the KL divergence between the two one-block conditional
distributions. For $t\in[0,1]$, let
$\bz_b^{(t)}:=\bz_b^{(0)}+t\Delta$. Every $v_j\in V\setminus\{v_1\}$
has Hamming weight at least $k$. Hence
\[
\bz_b^{(t)\top}(v_1-v_j)
=
(R-\eta+t\eta)d_H(v_1,v_j)
\ge
\frac{kR}{2}.
\]
Comparing the mass of $v_j$ to that of $v_1$ gives
\[
\mu_{\bz_b^{(t)}\mid V}(v_j)
\le
\exp(-kR/2)
\qquad
\text{for all }j\neq1,\ t\in[0,1].
\]
Since the conditional law on $V$ is an exponential family, the integral
form of Taylor's theorem gives
\[
\begin{aligned}
D_{\mathrm{KL}}\left(
\mu_{\bz_b^{(0)}\mid V}
\middle\|
\mu_{\bz_b^{(1)}\mid V}
\right)
&=
\int_0^1(1-t)
\Var_{\mu_{\bz_b^{(t)}\mid V}}(\Delta^\top X)\,dt\\
&\le
\frac12
\sup_{t\in[0,1]}
\E_{\mu_{\bz_b^{(t)}\mid V}}
\left[(\Delta^\top(X-v_1))^2\right].
\end{aligned}
\]
There are three alternatives to $v_1$, and for each of them
$|\Delta^\top(v_j-v_1)|\le2k\eta$. Therefore
\begin{equation*}
D_{\mathrm{KL}}\left(
\mu_{\bz_b^{(0)}\mid V}
\middle\|
\mu_{\bz_b^{(1)}\mid V}
\right)
\le
6\eta^2k^2e^{-kR/2}.
\label{eq:lower-block-kl}
\end{equation*}
The same calculation, with the interpolation traversed in the opposite
direction, gives the identical bound for the reverse KL divergence.

Now let $B:=n/(2k)$. For
$\omega\in\{0,1\}^B$, define
\[
\bz^\omega
:=
\left(
\bz_1^{(\omega_1)},\ldots,\bz_B^{(\omega_B)}
\right),
\]
and let $P_\omega$ be the corresponding conditional distribution on
$S=V^B$. The block-product structure and
\eqref{eq:lower-block-kl} imply that neighboring
hypotheses satisfy
\begin{equation*}
D_{\mathrm{KL}}\left(
P_\omega^{\otimes N}
\middle\|
P_{\omega\oplus e_b}^{\otimes N}
\right)
\le
6N\eta^2k^2e^{-kR/2}.
\label{eq:lower-neighbor-kl}
\end{equation*}
Thus, if
\begin{equation*}
N
\le
c_0\frac{e^{kR/2}}{\eta^2k^2}
\label{eq:lower-small-sample}
\end{equation*}
for a sufficiently small universal constant $c_0>0$, Pinsker's
inequality bounds the total variation distance in
\eqref{eq:lower-neighbor-kl} by $1/4$.
By convexity of total variation, the same bound holds between the two
mixtures obtained by fixing $\omega_b=0$ and $\omega_b=1$ and averaging
uniformly over the remaining bits.

It remains to convert this testing obstruction into the claimed
high-probability estimation obstruction. Given an arbitrary estimator
$\widehat{\bz}$, define $\widehat\omega_b$ by choosing the closer of
$\bz_b^{(0)}$ and $\bz_b^{(1)}$ to the $b$-th block of
$\widehat{\bz}$. The two-point testing inequality and the preceding
mixture-TV bound imply, for uniform $\omega$,
\[
\E_\omega d_H(\widehat\omega,\omega)
\ge
\frac{3B}{8}.
\]
Since $d_H(\widehat\omega,\omega)\le B$, it follows that
\[
\Pr_\omega\left(
d_H(\widehat\omega,\omega)\ge\frac{B}{8}
\right)
\ge
\frac{2}{7}.
\]
Whenever block $b$ is decoded incorrectly, the triangle inequality and
the nearest-neighbor definition give
\[
\|\widehat{\bz}_b-\bz_b^{(\omega_b)}\|_2
\ge
\frac12\|\Delta\|_2.
\]
Consequently, on the preceding event,
\[
\|\widehat{\bz}-\bz^\omega\|_2^2
\ge
\frac{B}{8}\cdot\frac{\|\Delta\|_2^2}{4}
=
\frac{n\eta^2}{32}
=
2\epsilon^2.
\]
After averaging over $\omega$, there must therefore be at least one
$\omega$ for which the estimator has error greater than $\epsilon$ with
probability at least $2/7$. In particular, uniform success probability
$0.99$ is impossible under \eqref{eq:lower-small-sample}. Substituting
$\eta^2=64\epsilon^2/n$ into the negation of
\eqref{eq:lower-small-sample} proves
\[
N
\ge
c\frac{n e^{kR/2}}{k^2\epsilon^2}
\]
for a universal constant $c>0$.
\end{proofthmlower}

\end{document}